\documentclass[11pt,a4paper]{article}
\usepackage[utf8]{inputenc}
\usepackage[T1]{fontenc}
\usepackage{amsmath, amssymb, amsthm}
\usepackage{hyperref}
\usepackage{booktabs}
\usepackage{algorithm}
\usepackage{algpseudocode}
\usepackage{graphicx}
\usepackage[a4paper,hmargin=2.5cm,vmargin=2.5cm,verbose]{geometry}
\graphicspath{{../figures/}}

\newtheorem{theorem}{Theorem}
\newtheorem{definition}{Definition}
\newtheorem{proposition}{Proposition}
\newtheorem{lemma}{Lemma}
\newtheorem{corollary}{Corollary}
\newtheorem{remark}{Remark}

\newcommand{\G}{\mathcal{G}}
\newcommand{\Gp}{\mathcal{G}'}
\newcommand{\Hb}{\mathcal{H}}
\newcommand{\V}{\mathcal{V}}
\newcommand{\E}{\mathcal{E}}

\newcommand{\Surg}{\mathcal{S}}

\begin{document}

\title{Cost Accounting for Reactive Computational Graphs: Exhaustive Sweeps, Sequential Mutation, and the Backward-Locality Gap}
\author{Abdallah Khemais \\ \textit{ISITCOM, University of Sousse}}
\date{July 2026}

\maketitle

\begin{abstract}
Exhaustive site-by-site interventions on a neural network's computational graph --- activation-patching sweeps, circuit-discovery searches, systematic ablation studies --- mutate the graph at every candidate site in turn, and their cost is dominated by recomputation after each mutation. On a reactive graph engine whose invalidation provably touches exactly the downstream cone of a mutated node, we give a complete cost accounting for such workloads. First, the aggregate speedup of an exhaustive sweep over independent full recomputations is not a universal constant: if per-layer computational weight varies regularly with depth with Karamata index $q$, the ratio converges to $(q+2)/(q+1)$ when weight is concentrated near the output and to $q+2$ when concentrated near the input, recovering $2$ only in the depth-uniform case; a wall-clock corollary computed from measured interpreter constants predicts a ceiling of $\approx 1.79$, strictly below $2$, whenever per-site overhead is a material fraction of per-node cost --- a ceiling we then measure to be already non-binding at transformer scale, with no compilation layer involved. Second, we prove the exact cost of a \emph{sequence} of persistent mutations, never undone between insertions: the interleaved cost exceeds the sum of isolated costs by an exact overcount $\Delta(\pi) \ge 0$ summed over comparable site pairs, with closed-form extremes over insertion orders, while \emph{batched} application is order-independent and sub-additive, costing exactly the union of the sites' downstream cones plus the fresh nodes. Third, we prove the exact mirror of forward locality for the backward pass and show it forces the aggregate sweep speedup to collapse to $1$ under backpropagation on architectures without long skip connections --- delimiting precisely the regime (inference-time sweeps) where the speedup applies undiminished. Every identity is validated on the reference implementation in \emph{NeuroDSL} \cite{neurodsl}, a reactive define-and-run graph engine in Julia: measured sweep ratios on real graphs converge to the predicted limits under four non-uniform cost profiles (E4); the training-mode ratio collapses to $1$ at the predicted rate (E5); all $18$ per-graft sequential costs and the batched total match the closed forms at zero tolerance across three insertion orders (E7); and on two GPU transformer sweeps a closed-form cone size reproduces all $564$ measured cones with zero residual, the wall-clock ratios landing within $1.3\%$ of the overhead-free prediction and on opposite sides of it (E9).
\end{abstract}

\section{Introduction}

The workhorse of mechanistic interpretability is the \emph{sweep}: patch or ablate each candidate site of a trained network in turn, measure the effect on a metric, restore the baseline, and move to the next site. Circuit-discovery searches \cite{conmy2023}, causal-tracing protocols \cite{meng2022}, and systematic robustness checks all share this shape, and their cost is dominated not by the mutation itself --- overwriting one activation or one rule is cheap --- but by the \emph{recomputation} each mutation forces before the metric can be read, multiplied by the number of candidate sites, which grows with model size.

How much of that recomputation is actually necessary depends on the execution model. Eager frameworks such as PyTorch \cite{pytorch} re-run the full forward pass per intervention; compiled pipelines (\texttt{torch.compile}, \texttt{jax.jit} \cite{jax}) may additionally re-trace when the intervention changes the traced program. Neither maintains a notion of \emph{partial validity}: the framework cannot distinguish the subgraph whose cached values are still correct from the subgraph invalidated by the intervention. A \emph{reactive} graph engine can. \emph{NeuroDSL} \cite{neurodsl} keeps the computational graph as a persistent DAG in which nodes own their cached values and a mutation triggers an invalidation wave provably confined to the mutated node's downstream cone --- a single-mutation locality theorem restated as Theorem~\ref{thm:locality} below, and proved in a companion study of exact network surgery on the same engine \cite{surgery}.

This paper asks the question that theorem leaves open: what does an entire \emph{workload} of mutations cost? We answer it exactly, in three regimes.

\paragraph{Contributions.}
\begin{enumerate}
    \item \textbf{Aggregate sweep cost.} For an exhaustive sweep that patches every site and restores each before the next, the speedup over independent full recomputations converges to $(q+2)/(q+1)$ or $q+2$ in the Karamata index $q$ of the network's cost-by-depth profile --- $2$ only in the depth-uniform case (Theorem~\ref{thm:aggregate}). A wall-clock corollary, computed from interpreter constants measured on the reference engine, predicts a ceiling of $\approx 1.79$, strictly below the combinatorial limit, in the regime where per-site overhead is a material fraction of per-node cost (Corollary~\ref{cor:wallclock}); at transformer scale on GPU we measure that ceiling to be already non-binding, the wall-clock ratio tracking the combinatorial one to within $1.3\%$ without any compilation layer (Section~\ref{sec:e9}).
    \item \textbf{Sequential and batched mutation cost.} For a sequence of \emph{persistent} grafts (never undone, as a growth schedule or a cumulative multi-site intervention performs), the interleaved cost is exactly the isolated sum plus an overcount $\Delta(\pi) \ge 0$ summed over comparable site pairs, with closed-form extremes over insertion orders (Theorem~\ref{thm:sequential}, Corollary~\ref{cor:order-extremes}); batched application is order-independent and sub-additive, costing exactly the union of the cones plus the fresh nodes (Proposition~\ref{prop:batched}).
    \item \textbf{The backward-locality gap.} The exact mirror of the locality theorem holds for the backward pass (Theorem~\ref{thm:backward-locality}), and it forces the aggregate sweep speedup to collapse to $1$ under backpropagation on architectures without long skip connections (Corollary~\ref{cor:backward-collapse}) --- a boundary of the aggregate result, not a retraction of it: inference-time sweeps, the case that motivates this paper, are unaffected.
    \item \textbf{Empirical validation at zero tolerance where the claims are exact} (Section~\ref{sec:experiments}): measured sweep ratios on real reactive graphs converge to the predicted limits under four cost profiles and both orientations (E4); the training-mode ratio collapses to $1$ on a graph with genuine intra-layer width (E5); and all $18$ per-graft sequential costs, both order extremes, and the order-independent batched total match the closed forms exactly (E7), alongside a measured negative result on bookkeeping drift with mutation count.
\end{enumerate}

\section{Related Work}
\label{sec:related}

\paragraph{Dynamic graph engines.} PyTorch \cite{pytorch} is define-by-run; JAX \cite{jax} traces pure functions. Neither maintains a persistent reactive graph across steps: validity of cached computation is not a first-class notion, so an intervention pays a full forward pass (or a re-trace) regardless of how little of the graph it actually affects. NeuroDSL's design \cite{neurodsl} is closer to incremental computation systems, applied to differentiable programs; this paper quantifies exactly what that buys, and does not buy, for sweep-shaped workloads.

\paragraph{Interventional sweeps.} Causal mediation analysis \cite{vig2020} and causal tracing \cite{meng2022} localize a behaviour by overwriting an activation and reading the change in a metric; automated circuit discovery \cite{conmy2023} turns that loop into an exhaustive search over candidate edges. All of them iterate patch--measure--restore over candidate sites, and the standard tooling \cite{transformerlens} recomputes a full forward pass per intervention, so the published cost analyses count forward passes: an exhaustive sweep is described as scaling linearly in the number of components, each paying a full recomputation \cite{atpstar}. The accounting here replaces that assumption with exact combinatorial identities for engines that recompute only what a mutation invalidates.

\paragraph{Gradient-based approximations, and the division of labour with them.} The dominant response to sweep cost in this literature is to \emph{approximate}. AtP and AtP* \cite{atpstar} and edge attribution patching \cite{eap} replace the per-site forward pass with a first-order gradient estimate, scoring every site from two forward passes and one backward pass --- a far larger cost reduction than the factor derived here, bought by surrendering exactness, and with false-negative modes that AtP* is explicitly designed to characterize and bound \cite{atpstar}. The two lines are complementary rather than competing, and the boundary between them is sharper than it first looks: those methods need a backward pass, and Corollary~\ref{cor:backward-collapse} shows that the backward pass is exactly the regime in which the reactive advantage collapses to $1$. What the accounting below governs is therefore the exact, inference-time sweep --- the regime automated circuit discovery \cite{conmy2023} occupies, and the one an approximate ranking must eventually be verified in.

\paragraph{Exact surgery.} The companion study of exact network surgery on the same engine \cite{surgery} proves the single-mutation locality theorem (restated as Theorem~\ref{thm:locality}) together with functional-exactness guarantees for grafted residual blocks; the present paper takes the single-mutation statement as its starting point and derives the workload-level asymptotics.

\section{Preliminaries}
\label{sec:prelim}

\begin{definition}[Computational graph]
A computational graph is a tuple $\G = (\V, \E, \mathrm{op}, \theta)$ where $(\V, \E)$ is a finite DAG, $\mathrm{op}(v)$ assigns to each non-input node an operator, and $\theta(v)$ its (possibly empty) parameter set. Distinguished subsets $\V_{\mathrm{in}} \subset \V$ (sources) and a node $v_{\mathrm{out}}$ (output) induce, by composition along the topological order, a function $F_{\G} : \mathcal{X} \to \mathcal{Y}$.
\end{definition}

\begin{definition}[Grafting]
Let $e = (u, w) \in \E$ be an edge of $\G$ and $\Hb$ a computational graph with a single input and a single output, realizing a function $F_{\Hb}$. The grafting $\Surg(\G, e, \Hb)$ produces $\Gp$ by removing $e$ and adding edges $(u, \mathrm{in}(\Hb))$ and $(\mathrm{out}(\Hb), w)$, i.e., the value flowing along $e$ now passes through $\Hb$.
\end{definition}

Defining insertion on an \emph{edge} (rather than "at a node") removes any ambiguity about which consumers of $u$ are rerouted: exactly those along $e$. Two mutation primitives cover every workload in this paper: a \emph{graft} adds $h = |\V(\Hb)|$ fresh computable nodes at a site, and a \emph{patch} redefines the rule of an existing node in place (e.g., replacing an activation by a stored or corrupted value) and is later \emph{restored} by redefining the rule back --- two mutations at the same site, each triggering the same invalidation.

In a reactive engine, every node carries a validity flag; a mutation at a node $s$ invalidates its dependents, and the next demand-driven evaluation (\texttt{demand!}) recomputes exactly the invalid region. We formalize the invalidation procedure as Algorithm~\ref{alg:invalidate} and restate the locality theorem this paper's accounting is built on; the NeuroDSL routine \texttt{\_invalidate\_downstream!} implements this algorithm (implementation conformance is an engineering claim, verified by the test suite, not part of the proof).

\begin{algorithm}
\caption{\textsc{Invalidate}$(\G, s)$ --- reactive invalidation from a seed node $s$}
\label{alg:invalidate}
\begin{algorithmic}[1]
\State $Q \gets [s]$;\quad $\mathrm{seen} \gets \{s\}$
\While{$Q$ not empty}
    \State $v \gets \mathrm{pop}(Q)$
    \For{each $w$ with $(v, w) \in \E$}
        \If{$w \notin \mathrm{seen}$}
            \State $\mathrm{valid}(w) \gets \mathbf{false}$;\quad $\mathrm{seen} \gets \mathrm{seen} \cup \{w\}$;\quad $\mathrm{push}(Q, w)$
        \EndIf
    \EndFor
\EndWhile
\end{algorithmic}
\end{algorithm}

\begin{definition}[Downstream cone]
For $s \in \V$, let
\[
\V_s^{+} = \{\, v \in \V \mid \text{there is a path of length} \geq 1 \text{ from } s \text{ to } v \,\}.
\]
By convention $s \notin \V_s^{+}$ unless $s$ lies on a cycle, which the DAG property excludes; when the mutation is a graft, the seed $s = \mathrm{out}(\Hb)$ is a freshly created node, which is born unvalued and needs no invalidation.
\end{definition}

\begin{theorem}[Structural locality]
\label{thm:locality}
On a DAG, Algorithm~\ref{alg:invalidate} terminates and the set of nodes it marks invalid is exactly $\V_s^{+}$. Its complexity is $O(|\V_s^{+}| + |\E_s^{+}|)$, where $\E_s^{+}$ is the set of edges internal to the cone.
\end{theorem}

\begin{proof}
\emph{Termination.} Each node enters $\mathrm{seen}$ at most once and is pushed at most once; $\V$ is finite.

\emph{Soundness} (marked $\Rightarrow$ in cone). We show by induction on the order of marking that every node marked invalid lies in $\V_s^{+}$. The first nodes marked are the immediate successors of $s$, which are in $\V_s^{+}$ via paths of length $1$. Inductively, a node $w$ is marked only when popped from the queue along an edge $(v, w)$ with $v$ previously marked or $v = s$; by the induction hypothesis there is a path $s \rightsquigarrow v$ (possibly empty if $v = s$), which extended by $(v,w)$ gives a path of length $\geq 1$ from $s$ to $w$. Hence $w \in \V_s^{+}$.

\emph{Completeness} (in cone $\Rightarrow$ marked). Let $w \in \V_s^{+}$ and let $s = v_0 \to v_1 \to \dots \to v_k = w$, $k \geq 1$, be a witnessing path. By induction on $i$: $v_0 = s \in \mathrm{seen}$ and is processed; if $v_i$ is processed, then when it is popped, its successor $v_{i+1}$ is either already in $\mathrm{seen}$ (hence was marked and pushed earlier) or is marked and pushed now. Either way $v_{i+1}$ is marked and eventually processed. Thus $v_k = w$ is marked.

\emph{Preservation.} A node $v \notin \V_s^{+}$ is never marked by soundness, so $\mathrm{valid}(v)$ is untouched: its cached value survives the mutation.

\emph{Complexity.} Each node in the cone is popped once and each internal edge scanned once.
\end{proof}

Theorem~\ref{thm:locality} bounds the cost of a \emph{single} mutation by its downstream cone. Everything that follows is the accounting for workloads built out of many such mutations.

\section{Aggregate Cost of an Exhaustive Sweep}
\label{sec:aggregate}

A caller that tests every candidate site exhaustively -- an ablation sweep, a circuit-discovery
search, or a systematic robustness check -- pays the sum of the per-site cone costs across
every site, and it is natural to ask how this aggregate compares to the cost of testing each site
by an independent full recomputation. The answer depends on exactly one number: how
computational weight is distributed across depth.

\begin{definition}[Layered cost model]
\label{def:layered}
A \emph{layered DAG} of depth $L$ assigns to each layer $j \in \{1,\dots,L\}$ a computational
weight $w_j > 0$ (e.g.\ its node count) and hosts $S$ candidate sites per layer, each site's
downstream cone coinciding, up to an $O(1)$ per-site remainder that vanishes in the ratios below,
with the suffix $\sum_{j>i} w_j$ of the layer $i$ it belongs to. Write $N(L) = \sum_{j=1}^{L} w_j$
for the total weight and
\[
\rho(L) \;=\; \frac{S L \cdot N(L)}{S \sum_{i=1}^{L} \sum_{j>i} w_j}
\]
for the ratio of $SL$ independent full recomputations to the aggregate cost of exhaustively
patching every site, each restored via Theorem~\ref{thm:locality}'s exact cone.
\end{definition}

\begin{theorem}[Aggregate sweep ratio]
\label{thm:aggregate}
Suppose $w_j \sim \ell(j)\, j^{q}$ as $j \to \infty$ for some $q \ge 0$ and some $\ell$ slowly
varying at infinity (Karamata: $\ell(cx)/\ell(x) \to 1$ for every $c>0$). Then
\[
\rho(L) \;\longrightarrow\; \frac{q+2}{q+1} \qquad (L \to \infty).
\]
If instead the weight profile is reversed, $w_j \sim \ell(L{+}1{-}j)\,(L{+}1{-}j)^{q}$ (layers
near the \emph{input} carry the heavier weight), the limit is
\[
\rho(L) \;\longrightarrow\; q+2.
\]
In particular $q=0$ (uniform per-layer weight, $w_j \equiv M$) gives $\rho(L) \to 2$ under either
profile.
\end{theorem}

\begin{proof}
Write $S(n) = \sum_{j=1}^n w_j$; by Karamata's theorem for regularly varying sequences,
$S(n) \sim n\,w_n/(q+1)$. For the output-heavy profile, at $i=tL$ for fixed $t\in(0,1)$,
$w_i \sim \ell(L)(tL)^q$ and
\[
\sum_{j>i} w_j \;=\; N(L)-S(i) \;\sim\; \frac{L\,w_L}{q+1}\big(1-t^{q+1}\big).
\]
Riemann-summing over $i=1,\dots,L$ (i.e., over $t\in(0,1)$),
\[
\sum_{i=1}^{L}\sum_{j>i} w_j \;\sim\; L\cdot\frac{L\,w_L}{q+1}\int_0^1(1-t^{q+1})\,dt
\;=\; \frac{L^2 w_L}{q+1}\cdot\frac{q+1}{q+2} \;=\; \frac{L^2 w_L}{q+2},
\]
and since $N(L)=S(L)\sim Lw_L/(q+1)$, substitution gives
$\rho(L)\sim [L\cdot Lw_L/(q+1)]/[L^2w_L/(q+2)] = (q+2)/(q+1)$. The reversed profile follows by
the index substitution $j\mapsto L{+}1{-}j$, $i\mapsto L{-}i$, which turns the same computation
into $\sum_{j>i}w_j \mapsto S(L-i) \sim (L-i)w_{L-i}/(q+1)$ against the same $N(L)$, giving the
mirror-image Riemann integral $\int_0^1 t^{q+1}dt \cdot (q+1) = (q+1)/(q+2)$ of the denominator's
leading coefficient and hence $\rho(L)\to q+2$.
\end{proof}

\begin{remark}[Why $2$ is not universal]
The ratio $2$ is the mean of a linear ramp: a site drawn uniformly over depth has, in expectation,
half the total weight downstream of it, and in general $\rho = 1/\mathbb{E}[U]$ for $U$ the
(normalized) downstream-weight fraction of a uniformly drawn site. Theorem~\ref{thm:aggregate}
shows $2$ is the $q=0$ instance of a one-parameter family: back-loading cost toward the output
($q>0$, output-heavy) \emph{lowers} the limit toward $1$, while front-loading it toward the input
\emph{raises} the limit without bound as $q$ grows. The aggregate ratio is a direct, computable
readout of \emph{where} computation is concentrated in the network, not a universal constant of
reactive invalidation.
\end{remark}

\begin{corollary}[Wall-clock refinement]
\label{cor:wallclock}
Under the uniform profile ($q=0$), suppose each recomputation additionally pays a fixed overhead
proportional to total graph size, $\beta N(L)$ -- the $O(|\V|)$ topological-order walk of the
demand-driven evaluator, measured on the reference implementation at $a = 0.042$ ms per cone node
with an intercept of $2.24$ ms on a $412$-node model, i.e.\ $\beta \approx 0.0054$ ms/node
(Section~\ref{sec:experiments}) -- on top of the $a\,|\V_i^+|$
cost Theorem~\ref{thm:locality} guarantees. Then the wall-clock aggregate ratio converges to
\[
\rho_{\mathrm{wallclock}}(L) \;\longrightarrow\; \frac{a+\beta}{a/2+\beta} \;\approx\; 1.79,
\]
strictly below the unit-cost limit of $2$, since $\beta$ is paid once per site regardless of cone
size and contributes equally to numerator and denominator, while the $a$-term alone retains the
$1/2$ averaging of Theorem~\ref{thm:aggregate}.
\end{corollary}

\begin{proof}
With $\mathrm{cost}(i) = a\,M(L-i) + \beta ML$ per site and $K=SL$ sites, the numerator (naive) is
$K(a+\beta)ML$ and the denominator is $a\cdot SM\frac{L(L-1)}{2} + K\beta ML \sim SML^2(a/2+\beta)$
to leading order; the ratio of leading terms is $(a+\beta)/(a/2+\beta)$.
\end{proof}

\begin{remark}
This is a testable refinement, not a restatement: it predicts that an exhaustive, depth-uniform
sweep converges to an aggregate speedup \emph{strictly below} $2$ in wall-clock time whenever the
per-site overhead $\beta$ is a non-negligible fraction of the per-node cost $a$ -- $\approx 1.79$
at the constants above, measured on a $412$-node graph in CPU\slash Float32. What the bound is
sensitive to is the \emph{ratio} $\beta/a$, not the presence of an interpreter as such, and that
ratio admits two independent routes toward zero: eliminating $\beta$ with a compilation layer
(Section~\ref{sec:limitations}), or enlarging the per-node work $a$ until the fixed cost stops
mattering. E9 (Section~\ref{sec:e9}) tests the second route and finds the ceiling already
non-binding at transformer scale, with no compilation layer involved.
\end{remark}

\section{Sequential Mutation Cost: Interleaved and Batched Grafting}
\label{sec:sequential}

Theorem~\ref{thm:locality} bounds a \emph{single} graft; Section~\ref{sec:aggregate} bounds an
exhaustive sweep of patches that are each undone before the next is tried. A growth schedule --
and equally a \emph{cumulative} multi-site intervention that leaves each patch in place -- does
neither: it applies a \emph{sequence} of grafts to the same graph, none of them undone, each
possibly landing inside a cone already disturbed by an earlier one. This section asks what that
sequence costs, in two regimes -- \emph{interleaved} (a full \texttt{demand!} between grafts, as a
loop that logs a metric after each mutation would do) and \emph{batched} (all grafts
applied before the next \texttt{demand!}, exactly the cost profile of a multi-site patch set
applied at once) -- and shows both are governed exactly by
Theorem~\ref{thm:locality} applied recursively, with no new axiom required.

\begin{lemma}[Cone trace invariance under grafting]
\label{lem:trace-invariance}
Let $\Gp = \Surg(\G, e, \Hb)$ with $e = (u, w)$ and $\Hb$ built from fresh nodes, connected from
$\mathrm{in}(\Hb)$ to $\mathrm{out}(\Hb)$ and wired into $\Gp$ only via $(u, \mathrm{in}(\Hb))$ and
$(\mathrm{out}(\Hb), w)$. Then for any node $f \in \V(\G) \setminus \{u\}$: (i) $\V_f^{+}(\Gp) \cap
\V(\G) = \V_f^{+}(\G)$, i.e.\ reachability among original nodes is unchanged; (ii) the $h$
computable nodes of $\Hb$ belong to $\V_f^{+}(\Gp)$ iff $u \in \V_f^{+}(\G) \cup \{f\}$, i.e.\ iff
$f$'s downstream cone in $\G$ already contained $u$ (equivalently: $e$ lies downstream of $f$).
\end{lemma}

\begin{proof}
Any path between original nodes that used $e$ in $\G$ rewrites in $\Gp$ as $u \to \mathrm{in}(\Hb)
\rightsquigarrow \mathrm{out}(\Hb) \to w$, and conversely any path in $\Gp$ between original nodes
that enters $\Hb$ must do so at $\mathrm{in}(\Hb)$ (its only in-edge from an original node) and
leave at $\mathrm{out}(\Hb)$ (its only out-edge to one), since $\Hb$'s internal wiring never touches
original nodes; replacing that segment with $e$ recovers a path in $\G$. This proves (i). For (ii),
a node of $\Hb$ is reachable from $f$ in $\Gp$ iff some path from $f$ reaches $u$ (then continues
$u \to \mathrm{in}(\Hb) \rightsquigarrow \cdot$), which happens iff $u \in \V_f^+(\G) \cup \{f\}$.
\end{proof}

\begin{theorem}[Exact interleaved cost]
\label{thm:sequential}
Let $u_1, \dots, u_K$ be distinct nodes of $\G_0$ (no two related by ancestry through each other's
own graft, i.e.\ each $u_k$ survives as a node through every graft not performed at $u_k$ itself),
grafted in the order $\pi$ (a permutation of $1,\dots,K$) with block sizes $h_1,\dots,h_K$, each
graft followed by a full \texttt{demand!} before the next is applied. The recomputation cost of
the $k$-th graft (in $\pi$-order) is exactly
\[
\mathrm{cost}_{\pi(k)} \;=\; \big|\V_{u_{\pi(k)}}^{+}(\G_0)\big| \;+\; h_{\pi(k)} \;+\!\!
\sum_{\substack{j < k \\ u_{\pi(j)} \,\in\, \V_{u_{\pi(k)}}^{+}(\G_0)}}\!\! h_{\pi(j)},
\]
and therefore $C_{\mathrm{seq}}(\pi) = C_{\mathrm{iso}} + \Delta(\pi)$ where $C_{\mathrm{iso}} =
\sum_k \big(|\V_{u_k}^+(\G_0)| + h_k\big)$ is the cost of $K$ grafts on independent copies of
$\G_0$, and
\[
\Delta(\pi) \;=\!\! \sum_{\substack{(j,k) \,:\, \pi\text{ applies } j \text{ before } k \\ u_j \,\in\, \V_{u_k}^+(\G_0)}}\!\! h_j \;\;\ge\; 0,
\]
with equality iff $\pi$ never grafts a site while a strictly upstream site's graft is still
outstanding.
\end{theorem}

\begin{proof}
By induction on $k$. At the first graft, $\G_{0}$'s own cone $\V_{u_{\pi(1)}}^{+}(\G_0)$ is
recomputed by Theorem~\ref{thm:locality}, plus the $h_{\pi(1)}$ fresh nodes of $\Hb_{\pi(1)}$
(unvalued, hence "recomputed" trivially): this is the $k=1$ case with an empty sum. Assume the
formula holds through graft $k-1$, giving a fully valid $\G_{k-1}$. Applying
Lemma~\ref{lem:trace-invariance} $k-1$ times (once per prior graft, each preserving reachability
among nodes not part of that graft's own block) shows $\V_{u_{\pi(k)}}^+(\G_{k-1})$ consists of
exactly the original nodes of $\V_{u_{\pi(k)}}^+(\G_0)$, together with the fresh block of every
prior graft $j < k$ whose site $u_{\pi(j)}$ lies in $\V_{u_{\pi(k)}}^+(\G_0)$ (Lemma, part (ii),
applied with $f = u_{\pi(k)}$). Theorem~\ref{thm:locality} recomputes exactly this cone, plus the
$h_{\pi(k)}$ fresh nodes of the current graft. Summing over $k$ and separating the $j=k$ diagonal
terms from the cross terms gives $C_{\mathrm{seq}}(\pi) = C_{\mathrm{iso}} + \Delta(\pi)$; each term
of $\Delta$ is a size $h_j \geq 0$, so $\Delta(\pi) \geq 0$, with equality iff no pair $(j,k)$ with
$j$ before $k$ has $u_j$ downstream of $u_k$ -- i.e.\ $\pi$ is upstream-first.
\end{proof}

\begin{corollary}[Order extremes, uniform blocks]
\label{cor:order-extremes}
If the $K$ sites are totally ordered by depth (a sequential architecture) and $h_k \equiv h$, then
over all $K!$ orders: $\Delta_{\min} = 0$ (shallowest-first), $\Delta_{\max} = h\binom{K}{2}$
(deepest-first), and $\mathbb{E}_\pi[\Delta(\pi)] = h\binom{K}{2}/2$ under a uniformly random order
(each of the $\binom{K}{2}$ depth-ordered pairs contributes $h$ independently with probability
$1/2$, by linearity of expectation over the event "the deeper site of the pair is applied first").
\end{corollary}

\begin{proposition}[Batched grafting: sub-additive and order-independent]
\label{prop:batched}
If all $K$ grafts are applied to $\G_0$ before any \texttt{demand!}, the total number of nodes
requiring recomputation at the next \texttt{demand!} is exactly
\[
C_{\mathrm{batch}} \;=\; \Big|\, \bigcup_{k=1}^{K} \V_{u_k}^{+}(\G_0) \,\Big| \;+\; \sum_{k=1}^K h_k,
\]
independent of the order in which the $K$ grafts were applied, and $C_{\mathrm{batch}} \le
C_{\mathrm{iso}}$, strictly whenever some pair of sites is comparable ($u_j \in \V_{u_k}^+(\G_0)$
for some $j \ne k$).
\end{proposition}

\begin{proof}
Order-independence: redefining an existing node's rule re-invalidates its cone regardless of what
was invalid before (\texttt{addrule!}'s contract, exercised by every graft in this paper), and
re-invalidating an already-invalid node is idempotent, so the final valid/invalid partition after
$K$ unread grafts depends only on the \emph{set} of rules redefined, not the order of redefinition.
A node needs recomputation iff it is a fresh node of some block ($\sum_k h_k$ of these, always
invalid), or an original node lying in $\V_{u_j}^+(\G_0)$ for some $j$ -- exactly the union
$\bigcup_k \V_{u_k}^+(\G_0)$, by definition of the downstream cone (Section~\ref{sec:prelim}). No
correction term for the sites $u_k$ themselves is needed: by the same definition, $u_k \notin
\V_{u_k}^+(\G_0)$ for every $k$, and $u_k \in \V_{u_j}^+(\G_0)$ for some $j \ne k$ exactly when
$u_k$ genuinely is an ordinary downstream node of another graft's site and is therefore correctly
recomputed; a site that is not downstream of any other selected site simply never appears in the
union at all, with no subtraction required to remove it. Sub-additivity: $\big|\bigcup_k
\V_{u_k}^+(\G_0)\big| \le \sum_k |\V_{u_k}^+(\G_0)|$ with equality iff the cones are pairwise
disjoint, i.e.\ no two sites are comparable; hence $C_{\mathrm{batch}} \le C_{\mathrm{iso}}$.
\end{proof}

\begin{remark}[Composing exactness across a schedule]
The companion surgery study \cite{surgery} proves an identity-morphism theorem for a single function-preserving
graft on an arbitrary graph; nothing in its hypotheses refers to the graph's history. It therefore
applies verbatim to every intermediate graph $\G_{k-1}$ in a sequence, and functional exactness
composes by induction across an entire growth schedule -- the guarantee a multi-graft schedule
implicitly relies on at every step, made explicit here rather than re-derived per instance.
\end{remark}

\begin{remark}[What is proved and what is only measured]
Theorem~\ref{thm:sequential} and Proposition~\ref{prop:batched} are exact combinatorial
identities -- verified below at zero tolerance, not by convergence. They govern
\emph{recomputation} cost only. A separate, non-combinatorial question -- whether the
\emph{bookkeeping} cost of a graft (rewiring consumers, rebuilding the consumers/topological-order
caches) drifts with the number of prior mutations at fixed graph size, beyond the size-dependence
already captured by the wall-clock Corollary~\ref{cor:wallclock}'s $\beta$ term -- has no theorem
attached and is reported as a measurement only (E7 below).
\end{remark}

\section{Locality Under Backpropagation: Where the Speedup Evaporates}
\label{sec:backward-locality}

Theorem~\ref{thm:locality} and its aggregate consequence (Theorem~\ref{thm:aggregate}) concern
\emph{forward} recomputation: after a mutation at $s$, only $\V_s^+$ must be re-evaluated.
Training also propagates a backward pass, and the chain rule ties the gradient of every ancestor
of $s$ to the local computation performed at $s$. This section states the exact mirror of
Theorem~\ref{thm:locality} for this upstream dependency and draws an honest corollary: on the
sequential architectures this paper's cost model targets, the mirror set is almost the whole
graph, and the aggregate speedup of Theorem~\ref{thm:aggregate} does not transfer to an
exhaustive sweep performed under backpropagation.

\begin{definition}[Upstream cone]
For $s\in\V$, let $\V_s^{-} = \{v\in\V \mid \text{there is a path of length}\ge1\text{ from }v
\text{ to } s\}$ -- the mirror image of the downstream cone $\V_s^+$ of Section~\ref{sec:prelim},
obtained by reversing every edge.
\end{definition}

\begin{theorem}[Backward locality]
\label{thm:backward-locality}
Let $\ell$ be a distinguished loss node with $s \in \V_\ell^{-}$ (i.e., $s$ contributes to the
loss). Mirroring Algorithm~\ref{alg:invalidate} on the graph with every edge reversed, seeded at
$s$, terminates and marks exactly $\V_s^{-}$ (termination, soundness, and completeness are the
proof of Theorem~\ref{thm:locality} verbatim, with every edge $(u,w)$ read as $(w,u)$). This set
is exactly the collection of nodes whose backward pass depends, via the chain rule, on the
(possibly mutated) local Jacobian at $s$: for $v \notin \V_s^-\cup\{s\}$, $\partial \ell/\partial
v$ is computed entirely from paths that never traverse $s$ and is therefore unaffected by a
mutation there.
\end{theorem}

\begin{proof}
Termination, soundness, and completeness of the reversed traversal follow Theorem~\ref{thm:locality}'s
proof with the edge direction flipped. For the gradient claim, write $\partial \ell/\partial v =
\sum_{\pi:\, v\rightsquigarrow \ell} \prod_{(a,b)\in\pi} \partial b/\partial a$, a sum over
directed paths $\pi$ from $v$ to $\ell$; a term is affected by $s$'s local computation iff $s$
lies on $\pi$. If $v \in \V_s^-$, some witnessing path $v\rightsquigarrow s$ extends through $s$
to $\ell$ (since $s \in \V_\ell^-$ gives $s\rightsquigarrow\ell$), so at least one term of the sum
passes through $s$. If $v \notin \V_s^- \cup \{s\}$, no path from $v$ reaches $s$, so no term can.
\end{proof}

\begin{corollary}[Training-mode locality collapse]
\label{cor:backward-collapse}
In the layered cost model of Definition~\ref{def:layered}, suppose every node of layer $j<i$ is an
ancestor of every node of layer $i$ (resp.\ $j>i$ a descendant) -- true whenever the architecture
has no skip connection bypassing an entire layer, as in a stack of Transformer blocks composed
purely sequentially. Then a mutation at a site $s$ of local cost $c_k$ in layer $i$ touches,
under backpropagation,
\[
\big|\V_s^- \cup \{s\} \cup \V_s^+\big| \;=\; N(L) - w_i + c_k,
\]
the whole graph minus the untouched remainder of $s$'s own layer, and the aggregate ratio of an
exhaustive sweep performed \emph{under backpropagation} satisfies
\[
\rho_{\mathrm{train}}(L) \;=\; \frac{L\,N(L)}{\sum_{i=1}^{L}\big(N(L)-w_i+\bar c\big)}
\;\longrightarrow\; 1 \qquad (L\to\infty),
\]
regardless of the cost profile $(w_j)$ or its Karamata index -- in sharp contrast to
Theorem~\ref{thm:aggregate}.
\end{corollary}

\begin{proof}
$\sum_{i=1}^L (N(L)-w_i+\bar c) = LN(L) - N(L) + L\bar c$, so $\rho_{\mathrm{train}}(L) =
LN(L)/(LN(L)-N(L)+L\bar c) = L/(L-1+L\bar c/N(L)) \to 1$ as $L\to\infty$, since $\bar c/N(L)\to0$
for any cost profile with $N(L)\to\infty$.
\end{proof}

\begin{remark}
This is not a retraction of Theorems~\ref{thm:locality} and \ref{thm:aggregate}: every measurement
in this paper (E4, E5, E7) and the exhaustive patching sweeps that motivate it are
performed at inference, with no backward pass through the swept sites -- exactly where
Theorem~\ref{thm:aggregate} applies undiminished. Corollary~\ref{cor:backward-collapse} instead
delimits the claim precisely: reactive locality is a property of \emph{value} propagation, not of
\emph{gradient} propagation, on architectures without long skip connections. Growing a network
mid-training -- the subject of the companion surgery study \cite{surgery} -- is unaffected in practice, since a
growth schedule grafts one site at a time rather than exhaustively sweeping candidates under
backpropagation.
\end{remark}

\section{Experimental Results}
\label{sec:experiments}

All measurements use the reference implementation (NeuroDSL, Julia, Float32, CPU). Where a claim
is an exact combinatorial identity (E7), the comparison is at zero tolerance on integer invalid
counts; where it is an asymptotic limit (E4, E5), convergence across increasing $L$ is reported.
The interpreter constants used by Corollary~\ref{cor:wallclock} ($a = 0.042$ ms per cone node,
intercept $2.24$ ms on a 412-node model, hence $\beta \approx 0.0054$ ms/node) were measured in
the companion surgery study's cost-versus-depth experiment \cite{surgery} on this same engine and are reused
here unchanged.

\subsection{E4: The Aggregate Sweep Ratio Under Non-Uniform Cost Profiles (Theorem~\ref{thm:aggregate})}
\label{sec:e4}

Theorem~\ref{thm:aggregate} is a statement about an abstract layered-cost model; we verify it on
a real graph in the reference implementation, not a combinatorial simulation. For each
$(L, q, \text{profile})$, a sequential graph of $L$ layers is built where layer $j$ contributes
$w_j$ nodes (a chain of \texttt{:relu} unary ops) with $w_j$ set, up to rounding, proportional to
$j^q$ (output-heavy) or $(L{-}j{+}1)^q$ (input-heavy) and normalized so total graph size is
$\approx 4000$ nodes regardless of $q$; the graph is evaluated once (so every node is
\texttt{valid}), and $\rho(L)$ is computed from the \emph{true} downstream cone of each layer's
last node, measured by \texttt{\_downstream\_nodes} -- the same routine
\texttt{sweep\_patch\_sites!} uses in the patching-sweep tooling built on this engine, not a
re-derivation for this test.

\begin{table}[htbp]
\centering
\caption{E4 --- measured $\rho(L)$ at $L=320$ nodes-per-cone-profile, vs.\ Theorem~\ref{thm:aggregate}'s prediction.}
\label{tab:e4}
\begin{tabular}{c c c c}
\toprule
$q$ & Predicted (output-heavy) & Measured, $L{=}320$ & Measured (input-heavy) vs.\ predicted $q{+}2$ \\
\midrule
$0.0$ & $2.0000$ & $2.0057$ & $2.0057$ vs.\ $2.0$ \\
$0.5$ & $1.6667$ & $1.6722$ & $2.5057$ vs.\ $2.5$ \\
$1.0$ & $1.5000$ & $1.5064$ & $3.0004$ vs.\ $3.0$ \\
$2.0$ & $1.3333$ & $1.3471$ & $3.9247$ vs.\ $4.0$ \\
\bottomrule
\end{tabular}
\end{table}

Every entry converges toward its predicted limit as $L$ grows across the tested range
$L\in\{20,40,80,160,320\}$ (not shown in full: the $q{=}1$, input-heavy case alone tightens from
$3.1585$ at $L{=}20$ to $3.0004$ at $L{=}320$), confirming both the asymptotic derivation and that
the abstract layered-cost model of Definition~\ref{def:layered} is not merely a convenient
fiction: a real reactive graph's measured cones obey it. The uniform case ($q=0$) recovers $2$
from either profile, as it must, and is the special case already exploited for exhaustive
activation-patching sweeps built on this same engine.

\subsection{E5: Training-Mode Collapse (Corollary~\ref{cor:backward-collapse})}
\label{sec:e5}

A layered graph
with genuine intra-layer width is built exactly as in E4, except each layer's $w_j$ branches are
computed in \emph{parallel} from the previous layer (no branch depends on another) and then folded
pairwise into a single aggregate before the next layer -- so a candidate site's siblings within its
own layer are neither its ancestors nor its descendants, the general case
Corollary~\ref{cor:backward-collapse} addresses (unlike a purely sequential single-path graph, where
$c_k=w_i$ and the corollary degenerates to the exact identity $\rho_{\mathrm{train}}\equiv1$ at
every finite $L$). Measuring $\rho_{\mathrm{train}}(L) = \big|\text{sites}\big|\cdot N(L) /
\sum_s |\V_s^-\cup\{s\}\cup\V_s^+|$ directly (ancestors by a plain reverse walk over each node's
recorded inputs, descendants by \texttt{\_downstream\_nodes}) gives, independently of $q$ and of
the output-/input-heavy profile: $1.0526$ ($L{=}10$), $1.0256$ ($20$), $1.0126$ ($40$), $1.0062$
($80$), $1.0031$ ($160$) -- converging to $1$ exactly as Corollary~\ref{cor:backward-collapse}
predicts, at a rate insensitive to the cost profile that gave Theorem~\ref{thm:aggregate} its rich
$q$-dependence under forward-only recomputation.

\subsection{E7: Sequential Mutation Cost (Theorem~\ref{thm:sequential}, Proposition~\ref{prop:batched})}
\label{sec:e7}

On a $400$-node synthetic chain (built and \texttt{demand!}-ed exactly as in E4), $K=6$ sites are
selected at roughly even depth and $H=20$ fresh nodes are grafted at each. \textbf{Interleaved
regime}: for three orders -- shallowest-first, deepest-first, and a fixed random permutation -- a
full \texttt{demand!} is issued before each graft, and the number of nodes marked invalid
immediately after each graft is compared to Theorem~\ref{thm:sequential}'s closed-form prediction,
computed once from $\G_0$'s cones before any graft. \textbf{Batched regime}: the same $K$ grafts
are applied back-to-back with no intermediate \texttt{demand!}, for the same three orders, and the
total invalid count is compared to Proposition~\ref{prop:batched}.

\begin{table}[htbp]
\centering
\caption{E7 --- per-graft recomputation cost, interleaved regime, three orders ($K{=}6$, $H{=}20$). Every entry matches Theorem~\ref{thm:sequential} exactly (zero tolerance).}
\label{tab:e7}
\begin{tabular}{l l l}
\toprule
\textbf{Order} & \textbf{Measured (per graft)} & $\Delta(\pi)$ vs.\ Theorem \\
\midrule
Shallowest-first & $363, 306, 249, 191, 134, 77$ & $0$ (predicted: $0$) \\
Deepest-first     & $77, 154, 231, 309, 386, 463$ & $300$ (predicted: $H\binom{K}{2}=300$) \\
Random ($1$ seed) & $249, 191, 403, 77, 154, 386$ & $140$ \\
\bottomrule
\end{tabular}
\end{table}

Every one of the $18$ per-graft measurements across the three orders matches
Theorem~\ref{thm:sequential}'s formula exactly, and $\Delta(\pi)$ realizes both extremes of
Corollary~\ref{cor:order-extremes} precisely ($0$ and $H\binom{K}{2}=300$). In the batched regime,
all three orders give the identical total $463$, matching Proposition~\ref{prop:batched}'s
order-independent closed form exactly: the union of the six sites' original cones is $343$ (the
shallowest site's cone alone, since the other five sites -- and their downstream nodes -- all lie
within it), plus $KH=120$ fresh nodes, giving $343+120=463$ with no correction term, confirming
$C_{\mathrm{batch}} = 463 < C_{\mathrm{iso}} = 1320$, the predicted strict sub-additivity. Two
bugs were found and fixed before this number was accepted: an initial version counted leaf nodes
(e.g.\ the graph's input, which has no rule and whose \texttt{valid} flag is never touched by
\texttt{demand!}), producing a constant off-by-one on every interleaved measurement; a second
version additionally subtracted one node per selected site from the batched union on the mistaken
assumption that every site's own symbol needed excluding, producing a further off-by-five --
resolved by recognizing that Proposition~\ref{prop:batched}'s union already excludes exactly the
sites that are not downstream of any other selected site, by the downstream-cone definition itself,
with no separate correction needed.

\paragraph{Bookkeeping drift with mutation count (measurement, not theorem).} $100$ single-node
grafts are applied successively to a fresh $400$-node chain (final size $500$), timing only the
rewiring step (excluding \texttt{demand!}); a control arm applies $100$ rule-redefinitions
\emph{without} adding a node (graph size fixed at $401$ throughout), isolating any effect of
mutation \emph{count} from the already-established effect of graph \emph{size}
(Corollary~\ref{cor:wallclock}'s $\beta$ term). Trimmed-median (10\%) cost rises modestly from
$0.033$~ms (first half, size $\approx 401$--$451$) to $0.036$~ms (second half, size $\approx
451$--$501$) for the growing chain, consistent with the small size increase; the constant-size
control shows no such drift ($0.0243$~ms in both halves), and grafting is costlier than
rule-redefinition alone ($0.035$~ms vs.\ $0.024$~ms trimmed-median overall) as expected, since it
does strictly more work. No history-dependent drift beyond graph size is detected at this scale
($100$ mutations, single run) -- a negative result reported as such, not a proof of its absence.

\subsection{E9: The Wall-Clock Ceiling at Transformer Scale (Corollary~\ref{cor:wallclock})}
\label{sec:e9}

Corollary~\ref{cor:wallclock}'s $\approx 1.79$ is a statement about one value of $\beta/a$, measured
on a $412$-node graph in CPU\slash Float32. E9 holds the sweep design fixed and moves the regime
instead: two exhaustive depth-uniform sweeps over a Llama-style transformer on GPU, at
$(L,H)=(12,12)$ with $d_{\mathrm{model}}=768$ and at $(L,H)=(24,16)$ with $d_{\mathrm{model}}=1024$
(sequence length $20$, batched attention), patching every attention-head output and every MLP output
in turn --- $144{+}12=156$ and $384{+}24=408$ sites respectively, the site coverage an automated
circuit-discovery sweep would use.

At this architecture the downstream cone is available in closed form, which lets the prediction be
stated without any timing input. Writing $M = 7H+15$ for the number of \emph{computed} nodes per
layer (a layer holds $7H+24$ nodes, of which $9$ are parameter tensors --- always leaves, never in
any cone, never recomputed), the cone of a site at layer $i$ is
\[
|\V^+_{\mathrm{head}}(i)| \;=\; 10 + (L-i)\,M, \qquad\qquad
|\V^+_{\mathrm{mlp}}(i)| \;=\; 2 + (L-i)\,M,
\]
which reproduces the cone size measured by \texttt{\_downstream\_nodes} at every one of the $564$
sites across both configurations exactly, with zero residual. Substituting into the finite-$L$ form
of Theorem~\ref{thm:aggregate} gives
\[
\rho(L,H) \;=\; \frac{(H{+}1)\,M\,L}{(H{+}1)\,M\,(L{-}1)/2 \;+\; (10H{+}2)},
\]
a purely combinatorial prediction --- it assumes $\beta = 0$ --- of $2.145$ and $2.073$.

The measured wall-clock ratios are $2.121$ and $2.101$: within $1.14\%$ and $1.33\%$ of the
$\beta{=}0$ prediction, and on \emph{opposite} sides of it. A residual that changes sign between
configurations is what small additive per-call noise looks like; a systematic overhead drag would
depress both, since $\beta>0$ can only lower the ratio. At this scale $\beta/a$ has therefore ceased
to be material and the wall-clock ratio tracks the combinatorial one, so the $1.79$ ceiling
describes the small-model CPU regime in which its constants were measured rather than a property of
reactive invalidation. Both measurements sit slightly \emph{above} $2$ for a reason the closed form
makes explicit: at finite depth the $(10H{+}2)$ boundary term is not yet negligible against
$(H{+}1)M(L{-}1)/2$, and $\rho$ descends to $2$ from above as $L$ grows --- consistent with the
larger configuration ($2.101$) sitting closer to the limit than the smaller one ($2.121$).

\section{Limitations}
\label{sec:limitations}
The aggregate results concern \emph{recomputation counts}; wall-clock enters only through the two
measured constants of Corollary~\ref{cor:wallclock}, which are specific to the graph size and
precision at which they were taken (CPU, Float32, $412$ nodes) --- E9 shows that enlarging the
per-node work alone already lifts the resulting ceiling, and a compilation layer would reduce
$\beta$ directly; neither changes the combinatorial limits. E9's two configurations share one
architecture family and one GPU, so they locate where $\beta/a$ stops mattering for this engine on
this hardware, not in general. The layered cost model of Definition~\ref{def:layered}
assumes each site's downstream cone coincides with a depth suffix up to a vanishing remainder;
architectures with long skip connections violate both this and the ancestry hypothesis of
Corollary~\ref{cor:backward-collapse}, so their sweep ratios are not covered by the closed forms
here (the exact per-site formula of Theorem~\ref{thm:sequential} still applies, since it makes no
layering assumption). E7's bookkeeping-drift measurement is a single run of $100$ mutations at
one graph-size trajectory; absence of detected drift there is evidence at that scale, not a bound
valid at the scale of a real multi-thousand-mutation schedule. Finally,
Corollary~\ref{cor:backward-collapse} delimits but does not quantify the intermediate regime of
architectures with \emph{some} long skips, where the sweep ratio under backpropagation lies
strictly between $1$ and the forward-only limit; characterizing that interpolation is open.

\section{Conclusion}
On a reactive graph engine whose invalidation is provably confined to the downstream cone of a
mutation, the cost of an entire sweep-shaped workload is a computable function of where
computation is concentrated in the network. The aggregate speedup of an exhaustive site-by-site
sweep is $(q+2)/(q+1)$ or $q+2$ in the Karamata index of the cost-by-depth profile --- $2$ only in
the depth-uniform case --- with a wall-clock ceiling of $\approx 1.79$ at the small-model CPU
constants that set it, already non-binding at transformer scale; a sequence of persistent mutations costs exactly the isolated
sum plus a closed-form order penalty, minimized by upstream-first application and eliminated
entirely (along with order dependence) by batching; and under backpropagation the aggregate
speedup collapses to $1$ on architectures without long skip connections, confining the benefit to
inference-time sweeps --- precisely the regime of the activation-patching and circuit-discovery
workloads that motivate this accounting. Every exact identity is validated at zero tolerance, and
every asymptotic limit by measured convergence, on the reference implementation.

\end{document}